\documentclass{article}
\usepackage{ijcai26}

\usepackage{times}
\usepackage{soul}
\usepackage{url}
\usepackage[hidelinks]{hyperref}
\usepackage[utf8]{inputenc}
\usepackage[small]{caption}
\usepackage{graphicx}
\usepackage{amsmath}
\DeclareMathOperator*{\argmin}{arg\,min}
\usepackage{amsthm}
\usepackage{booktabs}
\usepackage{algorithm}
\usepackage[switch]{lineno}
\usepackage{float}
\usepackage{placeins} % provides \FloatBarrier to stop floats drifting past a point
\usepackage{algpseudocode}
\usepackage{amssymb}
\usepackage{natbib}
\usepackage{bm}
\usepackage[dvipsnames]{xcolor}
\usepackage{multirow}
\usepackage{dblfloatfix}
\usepackage{microtype}
\usepackage{microtype}
\usepackage{placeins}
\usepackage{dblfloatfix}
\title{Amortizing Federated Adaptation: Hypernetwork Driven LoRA for Personalized Foundation Models}

\author{
Sunny Gupta$^1$
\and
Shambhavi Shanker$^1$\and
Amit Sethi$^1$\\
\affiliations
$^1$Indian Institute of Technology, Bombay\\
\emails
\{sunnygupta, 21d070066, asethi\}@iitb.ac.in,
}

\begin{document}

\maketitle

\begin{abstract}
    Federated fine-tuning of foundation models using Low-Rank Adaptation (LoRA) offers a communication efficient solution for distributed learning. However, existing federated LoRA methods suffer from two fundamental limitations: (1) structural aggregation bias, where independently averaging low rank factors fails to approximate the true combined update, and (2) client side initialization lag, as clients repeatedly reinitialize LoRA parameters across communication rounds, slowing convergence.

We propose \textbf{HyperLoRA}, a unified framework that addresses both issues through amortized federated adaptation through hypernetwork-driven LoRA generation and product space aggregation. Instead of iterative per-client optimization, HyperLoRA employs a learned generator that maps client distribution signatures to LoRA initializations, effectively amortizing per client adaptation. On the server side, we introduce a learned aggregation module that directly synthesizes updates in the low-rank product space, eliminating the inconsistencies of factor-wise averaging. A lightweight residual correction module further improves stability under heterogenous (non-IID) client distributions.

By replacing iterative optimization and heuristic averaging with learned operators, HyperLoRA jointly enables efficient personalization, unbiased aggregation, and faster convergence. Experiments on federated vision and vision-language benchmarks show that HyperLoRA achieves improved convergence speed, greater robustness to distribution shift, and stronger personalization performance compared to prior federated LoRA methods.
\end{abstract}
\section{Introduction}
\label{sec:intro}
 
Foundation models are increasingly deployed in decentralized, privacy-sensitive settings where direct fine-tuning is infeasible. Federated learning (FL)~\cite{mcmahan2017fedavg} addresses this by enabling collaborative training across clients without sharing raw data, and Low-Rank Adaptation (LoRA)~\cite{hu2022lora} makes federated fine-tuning of large models practical by training only a small set of low-rank parameters per layer~\cite{sun2024ffalora, wang2024flora}.
 
Despite this combination's appeal, existing federated LoRA methods suffer from two fundamental limitations.
 
\textbf{Aggregation bias.} Standard methods independently average the low-rank factors $\bm{A}_i, \bm{B}_i$ across clients, even though the true model update $\sum_i p_i \bm{B}_i \bm{A}_i$ lies in the \emph{product space} of these factors. The mismatch produces a systematic discrepancy between the aggregated update and the ideal one, which grows with client heterogeneity.
 
\textbf{Initialization inefficiency.} Clients reinitialize LoRA from scratch in every round, discarding information accumulated across rounds. The resulting warm-up phase delays convergence and inflates per-round local compute a particularly costly behavior on resource-constrained edge devices.
 
Our key insight is that federated adaptation has \emph{reusable structure}: similar client distributions yield similar low-rank adaptations. Rather than rediscovering this structure through repeated per-client optimization, the adaptation process itself can be amortized into learned operators.
 
We introduce \textbf{HyperLoRA}, a unified framework that replaces iterative optimization and heuristic averaging with learned, amortized operators. On the client side, a hypernetwork~\cite{ha2016hypernetworks} maps a compact distribution signature to a personalized LoRA initialization. On the server side, a learned synthesizer aggregates client updates directly in product space, with a lightweight residual corrector for stability under severe non-IID settings. This shifts federated adaptation from iterative optimization to learned generation, jointly addressing both bias and inefficiency within a single framework.
 
\paragraph{Contributions.}
\begin{itemize}
    \item We identify and formalize \emph{structural aggregation bias} in federated LoRA, and eliminate it via a learned product-space synthesizer with residual correction.
    \item We propose \emph{amortized client adaptation} through a hypernetwork conditioned on client distribution signatures, replacing repeated random initialization with personalized warm-starts.
    \item Across DomainNet and NICO++ on ViT and MLP-Mixer backbones, HyperLoRA outperforms state-of-the-art federated LoRA methods and matches their full-budget accuracy with $5\!\times$ fewer local iterations.
\end{itemize}

\section{Related Work}

\subsection{Federated Learning under Data Heterogeneity}

Federated Learning (FL) enables collaborative model training across distributed clients without sharing raw data. FedAvg~\cite{mcmahan2017fedavg} established the standard paradigm of local optimization followed by server-side parameter averaging, demonstrating strong communication efficiency. However, performance degradation under heterogeneous (non-IID) client distributions remains a major challenge.

Several methods have been proposed to improve optimization stability under heterogeneity. FedProx~\cite{li2020fedprox} introduces a proximal regularization term to limit excessive client drift, while SCAFFOLD~\cite{karimireddy2020scaffold} corrects local update variance using control variates. Personalized FL approaches further account for client-specific distributions by learning partially personalized models. Ditto~\cite{li2021ditto} jointly optimizes global and local objectives, FedAMP~\cite{huang2021fedamp} adaptively aggregates client models based on similarity, and FedBN~\cite{li2021fedbn} preserves client-specific batch normalization statistics to improve domain robustness. Additional heterogeneity-aware strategies explore adaptive aggregation, layer-wise personalization, and client-specific optimization mechanisms~\cite{tdinh2020pfedme, arivazhagan2019fedper}.

While these methods improve robustness in conventional FL settings, they primarily focus on full-model optimization and do not address the unique structural properties of parameter-efficient fine-tuning methods such as LoRA.

\subsection{Parameter-Efficient Fine-Tuning and LoRA}

The rapid growth of foundation models has made full fine-tuning computationally and memory intensive. Parameter-Efficient Fine-Tuning (PEFT) methods address this challenge by updating only a small subset of parameters while keeping the backbone frozen. Early approaches introduced trainable adapter modules~\cite{houlsby2019adapter}, bias-only fine-tuning~\cite{zaken2022bitfit}, prefix tuning~\cite{li2021prefix}, and prompt tuning~\cite{lester2021prompt}. These methods significantly reduce storage and communication overhead while retaining competitive performance.

Among PEFT approaches, LoRA~\cite{hu2022lora} has emerged as one of the most effective and widely adopted techniques. LoRA parameterizes weight updates as low-rank matrix decompositions, enabling efficient adaptation with minimal trainable parameters. Due to its strong empirical performance and compatibility with large-scale models, LoRA has been extensively applied across vision, language, and multimodal foundation models.

Recent works further extend PEFT techniques to large vision-language and multimodal systems, demonstrating that low-rank adaptation can effectively specialize foundation models while maintaining scalability~\cite{jia2022vpt, zhou2022coop, zhou2022cocoop, khattak2023maple}. Motivated by these advantages, our work focuses on federated LoRA adaptation under heterogeneous client distributions.

\subsection{Federated LoRA and Aggregation Bias}

Recent studies have explored combining LoRA with federated learning for communication-efficient fine-tuning of foundation models. FedIT~\cite{zhang2024fedit} is among the earliest approaches to integrate LoRA with FedAvg by training local LoRA modules and averaging them on the server. However, independently averaging low-rank factors introduces aggregation bias, since the averaged factors do not correspond to the true aggregated low-rank update.

Several subsequent works attempt to mitigate this issue. FFA-LoRA~\cite{sun2024ffalora} freezes one low-rank factor to ensure consistency between local and global updates, though this restricts adaptation flexibility and limits model capacity. FLoRA~\cite{wang2024flora} addresses aggregation bias by stacking and redistributing client LoRA modules, but incurs substantial communication overhead and raises privacy concerns by transmitting client-specific updates. Moreover, FLoRA repeatedly reinitializes local LoRA parameters across rounds, leading to inefficient optimization. Other approaches such as LoRA-FAIR~\cite{bian2025lorafair} and related adaptive aggregation strategies further improve robustness under heterogeneous settings, yet continue to rely on iterative optimization and approximate aggregation procedures.

In contrast to prior methods, HyperLoRA does not solely focus on correcting or approximating aggregation. Instead, it jointly addresses both aggregation bias and client initialization inefficiency through learned amortized operators. Our framework directly learns to synthesize global low-rank updates in product space while simultaneously generating client-specific LoRA initializations.

\subsection{Hypernetworks and Amortized Adaptation}

Hypernetworks~\cite{ha2017hypernetworks} generate model parameters using auxiliary neural networks and have been widely explored for efficient adaptation and personalization. In federated learning, pFedHN~\cite{shamsian2021pfedhn} employs hypernetworks to generate personalized client models conditioned on client representations, demonstrating improved adaptation under heterogeneous data distributions.

Related ideas also appear in meta-learning and amortized optimization, where a learned model predicts task-specific parameters or optimization trajectories instead of solving each task independently~\cite{finn2017maml, requeima2019cnaps}. Amortized inference methods similarly replace iterative optimization with learned inference networks that generalize across tasks.

More recently, context-conditioned LoRA generation methods, such as Doc-to-LoRA~\cite{charakorn2026d2l} and Text-to-LoRA~\cite{charakorn2025t2l}, generate LoRA parameters directly from contextual representations, enabling rapid task adaptation without explicit fine-tuning. These works demonstrate that low-rank adaptation itself can be learned as a parametric mapping.

However, existing hypernetwork and amortized adaptation methods have not been unified with federated LoRA aggregation. Prior approaches focus either on personalized parameter generation or on federated optimization, but do not jointly address client-specific initialization and aggregation bias in low-rank federated adaptation.

\subsection{Personalized Adaptation of Foundation Models}

Personalized federated learning aims to learn client-specific models that better capture heterogeneous local distributions while retaining the benefits of collaborative training. Existing approaches explore local fine-tuning, personalized heads, adaptive aggregation, and mixture-based personalization strategies~\cite{tan2022pflsurvey, collins2021fedrep}. 

With the emergence of foundation models, personalization has increasingly focused on parameter-efficient adaptation mechanisms that enable lightweight client specialization. Recent studies investigate client-specific prompt tuning, adapter learning, and LoRA-based personalization for large language and vision-language models~\cite{yi2023fedlora, yu2023fedpeft, ren2025fedfmsurvey}. These methods improve local adaptation while preserving communication efficiency.

Despite this progress, existing approaches typically treat personalization and aggregation as separate problems. HyperLoRA bridges this gap by integrating amortized client-specific generation with learned aggregation within a unified framework.
% =============================================================================
%  HyperLoRA — Methodology (NeurIPS-standard rewrite)
%
%  Required preamble (add to main.tex if not present):
%     \usepackage{amsmath, amssymb, amsthm, bm, algorithm, algpseudocode}
\newtheorem{proposition}{Proposition}
\newtheorem{remark}{Remark}
%     \DeclareMathOperator*{\argmin}{arg\,min}
% =============================================================================

\section{Methodology}
\label{sec:method}

\begin{figure*}[t]
    \centering
    \includegraphics[width=\linewidth]{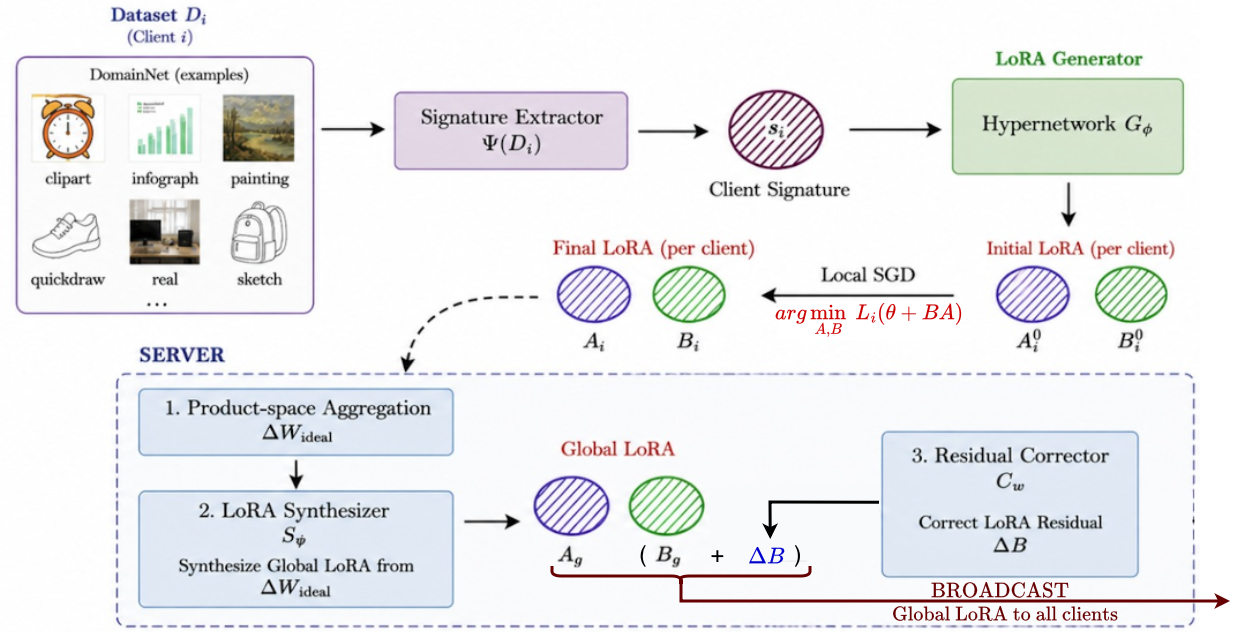}
    \caption{
    \textbf{HyperLoRA.} 
    Each client extracts a low-dimensional distribution signature $\bm{s}_i$ from its private dataset and feeds it to a shared hypernetwork generator $G_\phi$, which produces a personalized LoRA initialization $(\bm{A}_i^0,\bm{B}_i^0)$. After $E$ steps of local optimization, the resulting low-rank factors $(\bm{A}_i,\bm{B}_i)$ are uploaded. Rather than averaging factors independently which incurs structural bias (\S\ref{subsec:agg_bias}) the server applies a learned synthesizer $S_\psi$ that operates directly in the product space $\bm{B}\bm{A}$, followed by a residual corrector $C_\omega$ that compensates for low-rank projection error under severe heterogeneity.
    }
    \label{fig:hyperlora}
\end{figure*}

\subsection{Preliminaries}
\label{subsec:prelim}

\paragraph{Federated setup.}
We consider $K$ clients indexed by $i \in [K]$, each holding a private dataset $\mathcal{D}_i = \{(x,y)\}$ drawn from a client-specific distribution $\mathcal{P}_i$. We make no assumption that the $\mathcal{P}_i$ are identical; in fact, our design specifically targets the heterogeneous (non-IID) regime. A pre-trained foundation model $f_\theta$ is shared, with the backbone parameters $\theta$ frozen throughout federated training.

\paragraph{Low-rank adaptation.}
For each adapted linear layer $\bm{W}\!\in\!\mathbb{R}^{d\times d}$, LoRA~\citep{hu2022lora} introduces a rank-$r$ update $\Delta\bm{W} = \bm{B}\bm{A}$ with $\bm{B}\!\in\!\mathbb{R}^{d\times r}$, $\bm{A}\!\in\!\mathbb{R}^{r\times d}$, and $r \ll d$, yielding $\bm{W}' = \bm{W} + \bm{B}\bm{A}$. Only $(\bm{A},\bm{B})$ are trainable, reducing per-layer cost from $\mathcal{O}(d^2)$ to $\mathcal{O}(rd)$.

\paragraph{Federated objective.}
Let $p_i = |\mathcal{D}_i| / \sum_j |\mathcal{D}_j|$. The standard federated LoRA objective is
\begin{equation}
\label{eq:fl_obj}
\begin{aligned}
    \min_{\{\bm{A}_i,\bm{B}_i\}_{i=1}^K} \sum_{i=1}^K p_i\, \mathcal{L}_i(\theta + \bm{B}_i\bm{A}_i),
    \quad
    \\
\mathcal{L}_i(\theta+\bm{B}\bm{A}) \!=\! \mathbb{E}_{(x,y)\sim\mathcal{D}_i}\, \ell\!\left(f_{\theta+\bm{B}\bm{A}}(x),\,y\right).
\end{aligned}
\end{equation}
Throughout, we use the Frobenius norm $\|\cdot\|_F$ and reserve calligraphic symbols for losses, datasets, and statistics.

\subsection{Aggregation Bias in Product Space}
\label{subsec:agg_bias}

Existing federated LoRA methods~\citep{zhang2024fedit,sun2024ffalora,bian2025lorafair} aggregate the low-rank factors independently:
\begin{equation}
\label{eq:naive_agg}
    \bar{\bm{A}} = \sum_{i=1}^K p_i \bm{A}_i, \qquad \bar{\bm{B}} = \sum_{i=1}^K p_i \bm{B}_i,
    \qquad
    \Delta\bm{W}_{\text{avg}} \triangleq \bar{\bm{B}}\bar{\bm{A}}.
\end{equation}
The federated update consistent with~\eqref{eq:fl_obj}, however, is the weighted sum of \emph{client products}:
\begin{equation}
\label{eq:ideal}
    \Delta\bm{W}^{\text{ideal}} \triangleq \sum_{i=1}^K p_i\, \bm{B}_i\bm{A}_i.
\end{equation}
The discrepancy between~\eqref{eq:naive_agg} and~\eqref{eq:ideal} is not a numerical artifact; it is structural.

\begin{proposition}[Structural aggregation bias]
\label{prop:bias}
For any $K \geq 2$ and weights $\{p_i\}$ with $\sum_i p_i = 1$, the factor-averaged update admits the decomposition
\begin{equation}
\label{eq:decomp}
    \bar{\bm{B}}\bar{\bm{A}} = \underbrace{\sum_{i=1}^K p_i^2\, \bm{B}_i\bm{A}_i}_{\text{rescaled diagonal}} \;+\; \underbrace{\sum_{i\neq j} p_i p_j\, \bm{B}_i\bm{A}_j}_{\text{cross-client interactions}},
\end{equation}
and the aggregation error is bounded below by
\begin{equation}
\label{eq:bias_bound}
\begin{aligned}
    \epsilon_{\text{agg}} \triangleq \big\|\Delta\bm{W}^{\text{ideal}} - \bar{\bm{B}}\bar{\bm{A}}\big\|_F
    \;\geq\; \Big\| \sum_{i\neq j} p_i p_j\, \bm{B}_i\bm{A}_j \Big\|_F 
    \\
    - \sum_i p_i(1-p_i)\,\|\bm{B}_i\bm{A}_i\|_F.
\end{aligned}
\end{equation}
\end{proposition}

\begin{proof}[Proof sketch]
Expanding $\bar{\bm{B}}\bar{\bm{A}} = (\sum_i p_i \bm{B}_i)(\sum_j p_j \bm{A}_j)$ yields~\eqref{eq:decomp}. Subtracting~\eqref{eq:ideal} and applying the reverse triangle inequality gives~\eqref{eq:bias_bound}.
\end{proof}

\begin{remark}[Why the cross-terms matter]
\label{rem:cross_terms}
The cross-client products $\bm{B}_i\bm{A}_j$ for $i\neq j$ are \emph{chimeric}: $\bm{A}_j$ is trained on $\mathcal{P}_j$ while $\bm{B}_i$ is trained on $\mathcal{P}_i$, so their composition does not correspond to any realizable client adaptation. As pairwise distribution divergence $\mathrm{TV}(\mathcal{P}_i,\mathcal{P}_j)$ grows, the sub-spaces spanned by $\{\bm{B}_i\}$ and $\{\bm{A}_j\}$ become misaligned and $\epsilon_{\text{agg}}$ grows correspondingly. This explains the empirical fragility of factor-wise averaging under non-IID heterogeneity.
\end{remark}

Proposition~\ref{prop:bias} motivates a different design principle: aggregation should be performed \emph{in the product space}, not on the factors. Equivalently, the server should learn an operator whose codomain is the realizable adaptation manifold $\{\bm{B}\bm{A} : \bm{A}\in\mathbb{R}^{r\times d},\,\bm{B}\in\mathbb{R}^{d\times r}\}$, rather than averaging blindly in the unconstrained factor space.

\subsection{HyperLoRA: Amortized Federated Adaptation}
\label{subsec:framework}

HyperLoRA replaces two iterative routines with learned operators. \emph{Client adaptation} traditionally implemented as repeated SGD from random initialization in every round is amortized by a hypernetwork generator $G_\phi$. \emph{Server aggregation} traditionally implemented as factor-wise averaging is amortized by a product-space synthesizer $S_\psi$ followed by a residual corrector $C_\omega$. The three operators are trained jointly via a shared meta-objective (\S\ref{subsec:training}).

Formally, HyperLoRA introduces:
% \begin{align*}
% G_\phi &: \mathbb{R}^m \to \mathbb{R}^{d\times r}\!\times\!\mathbb{R}^{r\times d} && \text{(client adaptation generator)} \\
% S_\psi &: \big(\mathbb{R}^{d\times r}\!\times\!\mathbb{R}^{r\times d}\!\times\!\mathbb{R}\big)^K \to \mathbb{R}^{d\times r}\!\times\!\mathbb{R}^{r\times d} && \text{(product-space synthesizer)} \\
% C_\omega &: \mathbb{R}^q \to \mathbb{R}^{d\times r} && \text{(residual corrector)}
% \end{align*}

\textbf{Client adaptation generator:}
\begin{equation}
G_\phi :
\mathbb{R}^m
\to
\mathbb{R}^{d\times r}\times\mathbb{R}^{r\times d}
\end{equation}

\textbf{Product-space synthesizer:}
\begin{equation}
S_\psi :
\big(
\mathbb{R}^{d\times r}\times
\mathbb{R}^{r\times d}\times
\mathbb{R}
\big)^K
\to
\mathbb{R}^{d\times r}\times\mathbb{R}^{r\times d}
\end{equation}

\textbf{Residual corrector:}
\begin{equation}
C_\omega :
\mathbb{R}^q
\to
\mathbb{R}^{d\times r}
\end{equation}

We now describe each in turn, motivating the design choice in each case against the obvious baseline.

\subsection{Client Signatures and Amortized Initialization}
\label{subsec:amortized}

\paragraph{Distribution signatures.}
Each client computes a compact signature $\bm{s}_i = \Psi(\mathcal{D}_i)\in\mathbb{R}^m$ summarizing its data distribution, where $\Psi$ aggregates (i) class-frequency moments, (ii) backbone-feature mean and covariance, and (iii) optional domain descriptors when available. Critically, $m \ll d$ and $\bm{s}_i$ is a \emph{statistic}, not a sample: no raw data leaves the client. We treat $\Psi$ as a fixed extractor in this work.

\paragraph{Hypernetwork generator.}
The generator $G_\phi$ maps signatures to client-specific LoRA initializations:
\begin{equation}
\label{eq:gen}
    (\bm{A}_i^0,\bm{B}_i^0) = G_\phi(\bm{s}_i),
\end{equation}
trained to approximate the per-client optimum
\begin{equation}
\label{eq:gen_target}
    G_\phi(\bm{s}_i) \;\approx\; \argmin_{\bm{A},\bm{B}}\, \mathcal{L}_i(\theta + \bm{B}\bm{A}).
\end{equation}
This is an instance of \emph{amortized inference}~\citep{ha2017hypernetworks,shamsian2021pfedhn}: the $K$-fold inner-loop optimization that defines~\eqref{eq:gen_target} is replaced by a single feed-forward pass through $G_\phi$.

\paragraph{Why amortization?} Compared to MAML-style meta-initialization~\citep{finn2017maml}, $G_\phi$ produces a \emph{client-conditional} initialization, exploiting heterogeneity rather than averaging it out. Compared to random initialization (the FedIT default), it eliminates the warm-up phase during which the zero-initialized $\bm{B}^0$ produces uninformative gradients (cf.\ the \emph{initialization lag} analysis in \citealt{bian2025lorafair}).

\paragraph{Local update.}
Starting from $(\bm{A}_i^0,\bm{B}_i^0)$, each client performs $E$ steps of local SGD,
\begin{equation}
\label{eq:local}
    (\bm{A}_i,\bm{B}_i) = \mathrm{LocalUpdate}(\bm{A}_i^0,\bm{B}_i^0;\,\mathcal{D}_i),
\end{equation}
and uploads the resulting factors. Communication cost per layer is $\mathcal{O}(rd)$, identical to vanilla federated LoRA.

\subsection{Product-Space Synthesis}
\label{subsec:productspace}

To eliminate the bias identified in Proposition~\ref{prop:bias}, the server applies a learned synthesizer
\begin{equation}
\label{eq:syn}
    (\bm{A}^g,\bm{B}^g) = S_\psi\!\left(\{\bm{A}_i,\bm{B}_i,p_i\}_{i=1}^K\right),
\end{equation}
whose output is constrained by construction to lie on the rank-$r$ adaptation manifold. We train $S_\psi$ to minimize the product-space discrepancy
\begin{equation}
\label{eq:syn_loss}
    \mathcal{L}_{\text{syn}}(\psi) \;=\; \big\|\, \bm{B}^g\bm{A}^g - \Delta\bm{W}^{\text{ideal}}\, \big\|_F^2,
\end{equation}
where $\Delta\bm{W}^{\text{ideal}}$ is computed exactly via~\eqref{eq:ideal}.

\paragraph{Why a learned synthesizer?} The most natural baseline is to apply truncated SVD to $\Delta\bm{W}^{\text{ideal}}$, as in FlexLoRA~\citep{bai2024flexlora}. SVD gives the optimal rank-$r$ approximation in the Frobenius sense for a \emph{single} round, but is (i) stateless across rounds, (ii) computationally expensive when applied per layer, and (iii) blind to the downstream task loss. $S_\psi$ instead amortizes the synthesis across rounds, sees the round-to-round dynamics during meta-training, and when combined with $\mathcal{L}_{\text{func}}$ below can trade Frobenius accuracy for behavioral fidelity.

\paragraph{Functional matching (optional).}
Frobenius matching ignores the input distribution. We optionally augment~\eqref{eq:syn_loss} with a behavioral consistency term:
\begin{equation}
\label{eq:func}
    \mathcal{L}_{\text{func}}(\psi) \;=\; \mathbb{E}_{x\sim\mathcal{D}}\!\left\|\, f_{\theta+\bm{B}^g\bm{A}^g}(x) \;-\; \sum_{i=1}^K p_i\, f_{\theta+\bm{B}_i\bm{A}_i}(x) \,\right\|_2^2,
\end{equation}
which encourages the synthesized global model to match the ensemble of personalized clients in function space. The expectation is approximated using a small held-out probe set on the server.

\subsection{Residual Correction under Severe Heterogeneity}
\label{subsec:residual}

Even an ideally trained $S_\psi$ cannot recover the part of $\Delta\bm{W}^{\text{ideal}}$ that lies outside the rank-$r$ manifold. We define the residual
\begin{equation}
    \bm{R} \;\triangleq\; \Delta\bm{W}^{\text{ideal}} - \bm{B}^g\bm{A}^g \;\in\;\mathbb{R}^{d\times d},
\end{equation}
and absorb a low-rank correction back into the $\bm{B}$-factor:
\begin{equation}
\label{eq:residual}
    \Delta\bm{B} = C_\omega(\mathcal{S}), \qquad \bm{B}^g \leftarrow \bm{B}^g + \Delta\bm{B},
\end{equation}
where $\mathcal{S}\!\in\!\mathbb{R}^q$ collects aggregated round statistics: per-factor mean norms, factor variance, pairwise client cosine similarities, and effective participation. Conditioning $C_\omega$ on $\mathcal{S}$ rather than on raw factors keeps the corrector compact and prevents it from memorizing client-specific signals.

The corrector is trained via
\begin{equation}
\label{eq:corr_loss}
    \mathcal{L}_{\text{corr}}(\omega) \;=\; \big\|(\bm{B}^g+\Delta\bm{B})\bm{A}^g - \Delta\bm{W}^{\text{ideal}}\big\|_F^2 \;+\; \lambda\,\|\Delta\bm{B}\|_F^2,
\end{equation}
where the regularizer prevents $\Delta\bm{B}$ from drifting from the synthesized mean and preserves the consistent-initialization property identified by~\citet{bian2025lorafair}.

\paragraph{Asymmetric correction.} We apply $\Delta\bm{B}$ to the $\bm{B}$-factor only, leaving $\bm{A}^g$ untouched. This choice is consistent with the asymmetric role of LoRA factors observed in HydraLoRA~\citep{tian2024hydralora}: $\bm{A}$ acts as a shared projection that benefits from stability, while $\bm{B}$ carries client-specific adaptation that benefits from correction.

\subsection{Joint Training Objective and Algorithm}
\label{subsec:training}

We optimize $(\phi,\psi,\omega)$ jointly under a single meta-objective:
\begin{equation}
\label{eq:total}
\begin{aligned}
    \min_{\phi,\psi,\omega}\;\; \mathcal{L}_{\text{total}}
    \;=\;
    \underbrace{\sum_{i=1}^K p_i\,\mathcal{L}_i\!\left(\theta + \bm{B}_i\bm{A}_i\right)}_{\text{personalization}}
    \;+\; \alpha\,\mathcal{L}_{\text{syn}}\\
    \;+\; \beta\,\mathcal{L}_{\text{corr}}
    \;+\; \gamma\,\mathcal{L}_{\text{func}},
\end{aligned}
\end{equation}
with $(\alpha,\beta,\gamma)$ controlling the relative weight of aggregation fidelity, residual correction, and behavioral consistency. The personalization term flows through $G_\phi$ via the local-update Jacobian; we use a single-step unrolled approximation in practice.

HyperLoRA improves optimization efficiency through two complementary mechanisms.

First, the hypernetwork generator $G_\phi$ produces client-specific LoRA initializations conditioned on distribution signatures, replacing repeated random initialization with informed warm-starts. This reduces the early-round optimization instability commonly observed in federated LoRA training under heterogeneous client distributions.

Second, the product-space synthesizer eliminates the cross-client interaction terms introduced by independent factor averaging. By directly learning updates in the realizable low-rank product space, HyperLoRA avoids structurally inconsistent updates that accumulate under non-IID settings.

Together, these components reduce optimization variance at both the client and server levels, enabling faster convergence and improved robustness under heterogeneous federated adaptation.

\begin{table*}[!t]
\centering
\small
\setlength{\tabcolsep}{6pt}
\renewcommand{\arraystretch}{1.15}
\resizebox{\columnwidth}{!}{%
\begin{tabular}{cccc}
\toprule
\textbf{Local iters / round} & \textbf{LoRA-FAIR} & \textbf{HyperLoRA (ours)} & $\Delta$ \\
\midrule
10  & 68.47\% & \textbf{73.90\%} & \textcolor{ForestGreen}{$+5.43$} \\
20  & 71.99\% & \textbf{76.10\%} & \textcolor{ForestGreen}{$+4.11$} \\
50  & 73.83\% & \textbf{76.81\%} & \textcolor{ForestGreen}{$+2.98$} \\
100 & 74.81\% & \textbf{76.90\%} & \textcolor{ForestGreen}{$+2.10$} \\
\bottomrule
\end{tabular}%
}
\caption{
Accuracy vs.\ local compute budget (DomainNet, ViT-B/16). HyperLoRA at $20$ iterations exceeds LoRA-FAIR at $100$ full-budget accuracy with $5\!\times$ less per-round client compute.
}
\label{tab:compute_efficiency}
\end{table*}
\label{sec:experiments}

\begin{table*}[!t]
\centering
\small
\caption{
Feature non-IID results: average Top-1 accuracy across domains. $\Delta$ denotes HyperLoRA minus LoRA-FAIR.
}
\label{tab:feature_non_iid_avg}
\resizebox{\linewidth}{!}{
\begin{tabular}{llcccccccc}
\toprule
Dataset & Backbone & Centralized & FFA-LoRA & FedIT & FLoRA & FlexLoRA & LoRA-FAIR & HyperLoRA & $\Delta$ \\
\midrule
DomainNet & ViT        & 77.77 & 72.82 & 75.75 & 75.53 & 76.02 & 77.07 & \textbf{78.01} & $+0.94$ \\
DomainNet & MLP-Mixer  & 66.64 & 58.40 & 64.37 & 64.38 & 64.79 & 65.87 & \textbf{66.17} & $+0.30$ \\
NICO++    & ViT        & 91.51 & 90.28 & 90.58 & 90.93 & 90.60 & 91.24 & \textbf{92.91} & $+1.67$ \\
NICO++    & MLP-Mixer  & 84.50 & 80.05 & 82.51 & 82.29 & 83.08 & 83.56 & \textbf{85.86} & $+2.30$ \\
\bottomrule
\end{tabular}
}
\end{table*}

\begin{table*}[!t]
\centering
\small
\caption{
Feature+label non-IID results. We report average Top-1 accuracy across domains.
This is the harder heterogeneity setting. $\Delta$ denotes HyperLoRA minus LoRA-FAIR.
}
\label{tab:feature_label_non_iid_avg}
\resizebox{\linewidth}{!}{
\begin{tabular}{llcccccccc}
\toprule
Dataset & Backbone & Centralized & FFA-LoRA & FedIT & FLoRA & FlexLoRA & LoRA-FAIR & HyperLoRA & $\Delta$ \\
\midrule
DomainNet & ViT        & 77.77 & 72.82 & 73.89 & 74.25 & 74.25 & 74.99 & \textbf{76.11} & $+1.12$ \\
DomainNet & MLP-Mixer  & 66.64 & 52.88 & 60.77 & 59.62 & 61.20 & 62.28 & \textbf{63.14} & $+0.86$ \\
NICO++    & ViT        & 91.51 & 89.45 & 89.48 & 89.60 & 89.65 & 90.04 & \textbf{91.19} & $+1.15$ \\
NICO++    & MLP-Mixer  & 84.50 & 76.78 & 78.53 & 78.41 & 78.73 & 79.53 & \textbf{81.15} & $+1.62$ \\
\bottomrule
\end{tabular}
}
\end{table*}

\begin{algorithm}[t]
\caption{HyperLoRA: one communication round}
\label{alg:hyperlora}
\begin{algorithmic}[1]
\Require frozen backbone $\theta$; operators $G_\phi, S_\psi, C_\omega$; clients $\{\mathcal{D}_i\}_{i=1}^K$ with weights $\{p_i\}$
\For{each client $i \in [K]$ \textbf{in parallel}}
    \State $\bm{s}_i \gets \Psi(\mathcal{D}_i)$ \Comment{distribution signature}
    \State $(\bm{A}_i^0,\bm{B}_i^0) \gets G_\phi(\bm{s}_i)$ \Comment{amortized initialization}
    \State $(\bm{A}_i,\bm{B}_i) \gets \mathrm{LocalUpdate}(\bm{A}_i^0,\bm{B}_i^0;\mathcal{D}_i)$ \Comment{$E$ SGD steps}
    \State \textbf{upload} $(\bm{A}_i,\bm{B}_i)$ to server
\EndFor
\State $\Delta\bm{W}^{\text{ideal}} \gets \sum_i p_i\, \bm{B}_i\bm{A}_i$ \Comment{exact product-space target}
\State $(\bm{A}^g,\bm{B}^g) \gets S_\psi(\{\bm{A}_i,\bm{B}_i,p_i\})$ \Comment{learned synthesis}
\State $\bm{B}^g \gets \bm{B}^g + C_\omega(\mathcal{S})$ \Comment{residual correction}
\State update $(\phi,\psi,\omega)$ via gradient step on $\mathcal{L}_{\text{total}}$ \Comment{Eq.~\eqref{eq:total}}
\State \textbf{broadcast} $(\bm{A}^g,\bm{B}^g)$ to all clients
\end{algorithmic}
\end{algorithm}

\subsection{Communication and Complexity Analysis}

Table~\ref{tab:complexity} compares HyperLoRA with existing federated LoRA methods. HyperLoRA preserves the communication efficiency of standard low-rank federated adaptation while improving aggregation fidelity and initialization efficiency through learned amortized operators.

\begin{table}[h]
\centering
\small
\caption{Comparison of communication and optimization characteristics.}
\label{tab:complexity}
\begin{tabular}{lccc}
\toprule
Method & Upload Cost & Aggregation & Initialization \\
\midrule
FedIT & $\mathcal{O}(rd)$ & Factor Avg & Random \\
FFA-LoRA & $\mathcal{O}(rd)$ & Partial Shared & Random \\
FLoRA & $\mathcal{O}(Krd)$ & Stacking & Reinitialized \\
LoRA-FAIR & $\mathcal{O}(rd)$ & Adaptive Avg & Reinitialized \\
HyperLoRA & $\mathcal{O}(rd)$ & Product Synth & Hypernetwork \\
\bottomrule
\end{tabular}
\end{table}

\paragraph{Communication and computation.}
Per round, each client uploads $\mathcal{O}(rd)$ parameters per adapted layer, identical to vanilla federated LoRA~\citep{zhang2024fedit} and strictly less than FLoRA's $\mathcal{O}(Krd)$ broadcast~\citep{wang2024flora}. Server-side, $S_\psi$ and $C_\omega$ run in $\mathcal{O}(Krd + |\psi| + |\omega|)$, dominated by the synthesizer's attention over $K$ client tuples. The hypernetwork parameters $(\phi,\psi,\omega)$ are global and \emph{independent of $K$}, preserving scalability to large federations.

\begin{remark}[Variance reduction via amortization]
\label{rem:variance}
For any client $i$ with bounded loss curvature, the variance of the initialization $\bm{B}_i^0\bm{A}_i^0$ produced by $G_\phi(\bm{s}_i)$ is upper-bounded by the input-dependent variance of $G_\phi$ on the signature manifold, which is strictly smaller than the variance of i.i.d.\ Gaussian initialization whenever $G_\phi$ has been trained on signatures from non-degenerate distributions. This yields the empirically observed faster convergence in low-compute regimes (\S\ref{sec:experiments}).
\end{remark}

\section{Experiments}

% =============================================================================
%  HyperLoRA — Maximally Compressed Experiments (§5.1, §5.2, §5.3)
%  Prose: ~200 words total. Tables unchanged.
%  Aligned with TABLE data (HyperLoRA wins all four feature non-IID settings).
% =============================================================================

\subsection{Experimental Setup}
\label{subsec:setup}

We evaluate on \textbf{DomainNet} (six visual domains) and \textbf{NICO++} (six context domains) under feature non-IID and feature+label non-IID partitions, reporting Top-1 accuracy averaged across domains. Baselines: FFA-LoRA, FedIT, FLoRA, FlexLoRA, and LoRA-FAIR the strongest direct competitor, since it explicitly targets both aggregation bias and initialization lag. HyperLoRA uses rank $r{=}16$ and matches each baseline's round budget; DomainNet-ViT results are averaged over three seeds (std ${\approx}{\pm}0.065$).

\subsection{Main Results under Feature Non-IID}
\label{subsec:feature_non_iid}

Table~\ref{tab:feature_non_iid_avg} shows that HyperLoRA achieves the best average among federated methods across all four dataset-backbone combinations. It even exceeds the centralized reference in three of the four settings, while remaining slightly below centralized performance on DomainNet with MLP-Mixer. The largest margins over LoRA-FAIR occur on NICO++ ($+2.30$ on MLP-Mixer, $+1.67$ on ViT). The per-domain analysis in \S\ref{subsec:domain_analysis} shows gains concentrate in distribution-shifted domains, consistent with the Wasserstein-scaling prediction of Lemma.

\subsection{Communication and Compute Efficiency}
\label{subsec:efficiency}

Communication cost matches the cheapest baselines by construction. To test whether HyperLoRA also reduces \emph{client compute} the binding budget on edge devices we vary local iterations per round (DomainNet, ViT-B/16, 50 rounds; Table~\ref{tab:compute_efficiency}).
At 10 local iterations, HyperLoRA leads LoRA-FAIR by $5.43$ points; both methods improve with larger local budgets, but the performance gap narrows as LoRA-FAIR receives enough compute to reduce its warm-up disadvantage. Critically, HyperLoRA at $20$ iterations ($76.10\%$) already exceeds LoRA-FAIR at $100$ iterations ($74.81\%$), achieving stronger accuracy with $5\!\times$ less local compute, corresponding to an $80\%$ reduction in per-round client optimization, heterogeneity-aware initialization reduces the warm-up phase that delays randomly initialized LoRA adaptation. The trend replicates on NICO++ ($+1.73$ at iter${=}10$), with smaller gains reflecting its narrower domain spread.

% =============================================================================
%  OPTIONAL: Convergence-curve figure to accompany the table
%  Recommended placement: side-by-side with Table tab:compute_efficiency,
%  using \begin{figure}...\end{figure} or a minipage layout.
%
%  The figure should plot Top-1 accuracy vs. communication round for
%  HyperLoRA and LoRA-FAIR at iter=10 (low-compute) and iter=100 (full-budget).
%  Visual: HyperLoRA-iter10 should track or exceed LoRA-FAIR-iter100,
%  making the 5x-less-compute story immediate.
% =============================================================================

\subsection{Main Results under Feature+Label Non-IID}
\label{subsec:feature_label_non_iid}

Table~\ref{tab:feature_label_non_iid_avg} reports results under the more challenging feature+label non-IID setting. This setting is more representative of realistic FL deployments, where clients differ not only in visual/domain characteristics but also in label distributions.

HyperLoRA consistently outperforms all federated baselines across all four dataset-backbone combinations. Compared with LoRA-FAIR, it improves average Top-1 accuracy by $+1.12$ on DomainNet with ViT, $+0.86$ on DomainNet with MLP-Mixer, $+1.15$ on NICO++ with ViT, and $+1.62$ on NICO++ with MLP-Mixer. The strongest average gain therefore appears on NICO++ with MLP-Mixer, while DomainNet with ViT provides a high-accuracy setting where HyperLoRA improves substantially over the strongest federated baseline.

At the domain level, the improvement is especially pronounced on Quickdraw, where HyperLoRA improves over LoRA-FAIR by $+6.48\%$, and on Real by $+0.50\%$. This supports our hypothesis that HyperLoRA is most beneficial when standard federated LoRA aggregation struggles to reconcile heterogeneous client updates.

\subsection{Per-Domain Analysis: Where and Why HyperLoRA Wins}
\label{subsec:domain_analysis}
 
Per-domain accuracy reveals where HyperLoRA's amortized design pays off. On DomainNet under feature non-IID, the largest gains over LoRA-FAIR appear on \emph{Painting} and \emph{Quickdraw} two domains with substantial visual distribution shift. Under the harder feature+label non-IID setting, HyperLoRA consistently improves over LoRA-FAIR across all four dataset-backbone combinations, with average gains ranging from $+0.86$ to $+1.62$.
 
The mechanism is interpretable. The generator $G_\phi$ supplies each client with a distribution-conditioned warm-start that absorbs client-specific structure before local optimization, and the synthesizer $S_\psi$ consolidates the resulting heterogeneous adaptations directly in product space, avoiding the chimeric cross-terms (Prop.~\ref{prop:bias}) of factor-wise averaging. Gains are not uniform across all datasets and backbones, but they are largest when the domain and label shifts create stronger client-level heterogeneity. We therefore position HyperLoRA as a method designed for constrained, heterogeneous FL regimes, rather than a universal dominator across all benchmarks.

\section{Conclusion}
\label{sec:conclusion}
 
We introduced HyperLoRA, which recasts federated low-rank adaptation as amortized conditional generation. A hypernetwork generator replaces per-client iterative initialization, a learned product-space synthesizer replaces factor-wise averaging, and a residual corrector stabilizes aggregation under severe heterogeneity. Together these eliminate the structural aggregation bias and warm-up delay that limit prior federated LoRA methods. Experiments on DomainNet and NICO++ across ViT and MLP-Mixer backbones show consistent gains over state-of-the-art baselines, matching full-budget accuracy with $5\!\times$ fewer local iterations. Future work will extend HyperLoRA to language and multimodal foundation models, heterogeneous client ranks, and differentially private federated settings.

\FloatBarrier % prevent floats from slipping into/after the references
\appendix

\bibliographystyle{named}
\bibliography{ijcai26}

\end{document}